\documentclass{article}

\PassOptionsToPackage{numbers, compress}{natbib}
\usepackage[preprint]{neurips_2026}

\usepackage[utf8]{inputenc}
\usepackage[T1]{fontenc}
\usepackage{hyperref}
\usepackage{url}
\usepackage{booktabs}
\usepackage{amsfonts}
\usepackage{amsmath}
\usepackage{amssymb}
\usepackage{amsthm}
\usepackage{nicefrac}
\usepackage{microtype}
\usepackage[table]{xcolor}
\usepackage{algorithm}
\usepackage{algpseudocode}
\usepackage{graphicx}
\usepackage{subcaption}
\usepackage{multirow}
\usepackage{tikz}
\usetikzlibrary{positioning, arrows.meta, calc}

\newtheorem{theorem}{Theorem}
\newtheorem{proposition}[theorem]{Proposition}

\newtheorem{corollary}[theorem]{Corollary}

\title{Agentic-DPO: From Imitation to Agentic Policy Optimization on Expert Trajectories}

\author{
\textbf{Yixiong Chen},
\textbf{Alan Yuille} \\
Johns Hopkins University\\
\texttt{ychen646@jh.edu}
}

\begin{document}

\maketitle

\begin{abstract}
Large Language Model (LLM) agents are commonly trained from expert trajectories using supervised fine-tuning (SFT), which treats multi-turn agent behavior as ordinary text imitation. This recipe is simple and low-cost, but it only learns to imitate the sequence of expert actions, rather than training the agent to choose the right action against plausible mistakes at each state. Existing methods to mitigate this problem include preference learning or reinforcement learning, but they usually need high-cost environment rollouts and reward models. We propose \textbf{Agentic-DPO}, a lightweight offline agent policy optimization method that turns expert trajectories into state-conditioned preference supervision. At each expert action state, Agentic-DPO samples a one-step action from the current state, treats plausible wrong actions as negatives, and contrasts them with the expert action using a DPO-style preference objective. To avoid mixing both policy and schema in preference learning, we introduce \textbf{Policy-Preserving Augmentation} (PPA), which renders the same latent trajectory under multiple schemas while keeping the expert policy fixed. Agentic-DPO requires no online environment rollout, reward model, or full-trajectory student exploration. We conduct experiments across StableToolBench, $\tau$-bench retail, and Mind2Web, where Agentic-DPO consistently improves agents at different model scales beyond imitation. In particular, it raises $\tau$-bench accuracy from $21.7\%$ (SFT) to $41.4\%$ for a 9B model, matching online GRPO under the same backbone with only step-level rollouts and without environment interaction during gradient steps. The results suggest that expert trajectories can support low-cost agentic policy optimization when converted from demonstrations into state-level action preferences. Code for Agentic-DPO is released at \href{https://github.com/Schuture/Agentic-DPO}{https://github.com/Schuture/Agentic-DPO}.
\end{abstract}

%% ============================================================
\section{Introduction}
%% ============================================================

Large Language Model (LLM) agents are increasingly used to solve tasks that require interacting with tools, users, web pages, or software environments over multiple turns~\citep{toolbench,stabletoolbench,taubench,mind2web}. 
A common way to train such agents is to collect expert trajectories from stronger models or scripted policies, and then apply supervised fine-tuning (SFT) to imitate the full trajectory~\citep{toolace,atlas}. 
This recipe is attractive because it is simple, stable, and does not require online interaction with the environment. 
However, it inherits a basic limitation of behavior cloning: SFT only imitates the token sequences of the expert trajectory. 
% It tells the student what the expert did, but not what the student would have done incorrectly at the same state, nor how its likely mistakes differ from the expert action. 
% This missing contrast is exactly the signal that online reinforcement learning obtains through rollouts: the model acts, observes whether the action fails, and updates against its own behavior.

This limitation is especially important for agent training, because an agent trajectory is not merely a text response; it is a sequence of state-conditioned decisions. 
At each step, the model must decide whether to call a tool or respond to the user, which tool or action to choose, and what arguments or grounding targets to provide. 
Recent work has started to address this mismatch by showing that policy can be learned better by algorithms that utilizes the structure of expert trajectories.
A straightforward way is to improve SFT by selecting only valuable action steps to train, with heuristic rules (SWE-Lego~\citep{swelego}) or the help of external annotation models (ATLaS~\citep{atlas}).
More sophisticated methods combine trajectory structure with preference learning or reinforcement learning (RL) like error step annotation (Step-DPO~\citep{stepdpo}), action grouping (HPL~\citep{hpl}), local on-policy rollouts (PivotRL~\citep{pivotrl}), and test-based failure feedback (Agent-RLVR~\citep{agentrlvr}).
These methods point to an important direction: agent training should use the decision structure inside trajectories, especially where the current policy would fail. 
However, they often recover this signal through costly step selection~\citep{stepdpo,atlas}, executable environments~\citep{pivotrl,agentrlvr}, or even human feedback~\citep{stepdpo}. 
We ask whether expert trajectories alone can be converted into policy-level training signals at much lower cost.

\begin{figure*}[t]
    \centering
    \includegraphics[width=1.0\textwidth]{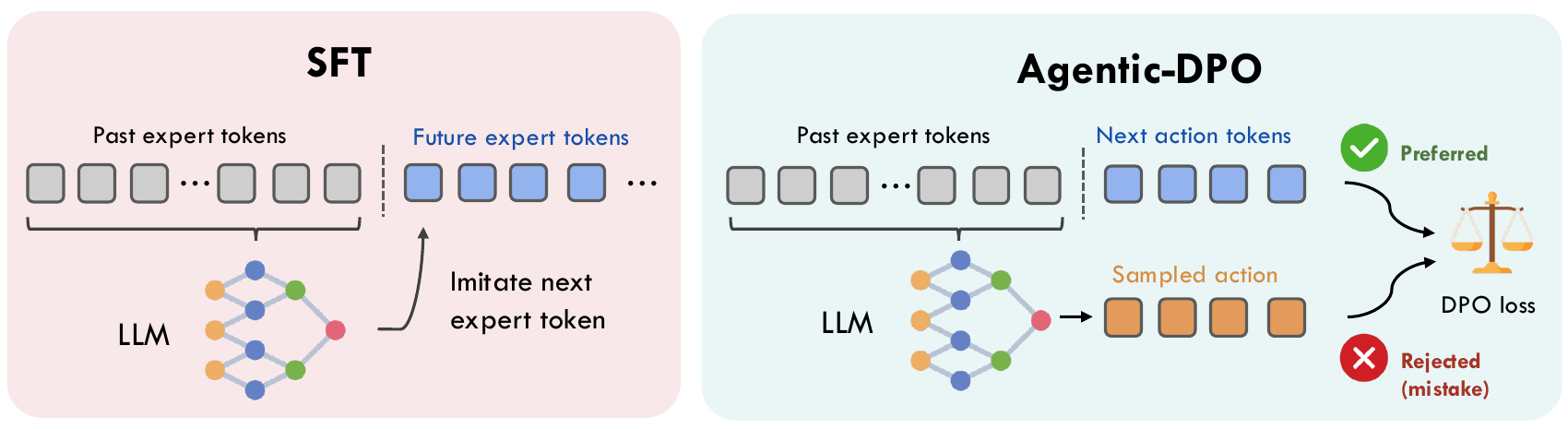}
    \caption{
    Illustration of the difference between SFT and Agentic-DPO for agent training.
    SFT treats an expert trajectory as token-level supervision and trains the model to imitate demonstrated tokens.
    In contrast, Agentic-DPO treats each expert action as a state-conditioned policy decision. It contrasts the expert action with a one-step action sampled from the current student policy, and trains the model to prefer the expert action over the likely student mistake.
    }
    \label{fig:agentic_dpo_overview}
\end{figure*}

We propose \textbf{Agentic-DPO}, an offline agent policy optimization method that turns expert trajectories into state-conditioned action preferences.
Figure~\ref{fig:agentic_dpo_overview} illustrates the core difference between Agentic-DPO and standard SFT.
While SFT treats the expert trajectory as token-level supervision and trains the model to imitate demonstrated tokens, Agentic-DPO treats each expert action as a \emph{local} policy decision.
Given an expert trajectory, we view each expert action as the preferred action at its state.
At the same state, we sample a one-step action from the current agent (namely student in the training) and use plausible wrong actions as negative samples.
The expert action and the student negative then form a preference pair for a preference learning objective. We adopt DPO-style loss~\citep{dpo}.
This gives the student direct feedback about its own likely mistakes: the model is not only trained to imitate the expert action, but also to assign lower probability to the wrong action it would have chosen at that state.
Because the rollout stops after the current action and does not execute the environment, Agentic-DPO avoids online environment interaction, reward-model training, and full-trajectory exploration.

Naively applying Agentic-DPO to agent action strings is still poorly conditioned. 
The action text usually mixes the latent policy decision and the schema-specific rendering of that decision, thus the preferred and dispreferred sequences may differ along both axis, but the DPO loss does not distinguish them.
To address this, Agentic-DPO uses two stabilizing designs. 
First, we keep an SFT anchor so that the model starts from valid expert behavior and does not need to discover the action format through preference learning alone. 
Second, we introduce \textbf{Policy-Preserving Augmentation} (PPA), which renders the same latent trajectory under multiple schemas while keeping the expert policy fixed. 
Across these views, the state-action decision is preserved, but the surface realization changes, making the Agentic-DPO preference direction better aligned with the latent agent policy.

We evaluate Agentic-DPO on three agent benchmarks StableToolBench~\citep{stabletoolbench}, $\tau$-bench retail~\citep{taubench}, and Mind2Web~\citep{mind2web} that cover complementary action spaces and interaction patterns. 
Our results show that Agentic-DPO improves small agents beyond imitation across all three settings. 
On StableToolBench, Agentic-DPO raises a Qwen3.5-2B from $57.1\%$ SFT canonical accuracy to $90.9\%$, surpassing strong agent training baselines.
On $\tau$-bench retail, Agentic-DPO with Qwen3.5-9B reaches $41.4\%$ success, outperforming GRPO under the same backbone ($40.0\%$) without online RL rollouts.
On Mind2Web, it achieves $64.4\%$ average step success across the cross-task, cross-website, and cross-domain held-out splits.
These results suggest that expert trajectories can support strong and low-cost agentic policy optimization with only additional action-level rollout.

Our contributions are summarized as follows: (1) We introduce Agentic-DPO, a lightweight offline policy optimization method that constructs state-conditioned action preferences by contrasting expert actions with one-step student negatives. (2) We introduce Policy-Preserving Augmentation, which renders the same latent trajectory under multiple schemas to reduce schema-dominated preference gradients and stabilize action-level DPO. (3) We show that Agentic-DPO improves small agents across controlled tool use, long-horizon user-interleaved interaction, and web-GUI grounding, without online environment rollouts or reward-model training.

%% ============================================================
\section{Related Work}
%% ============================================================

\paragraph{Trajectory tuning for LLM agents.}
Beyond prompting-based scaffolding~\citep{react,reflexion,guiagents,seeact,scribeagent}, open-source agents are typically trained by supervised fine-tuning over full expert trajectories~\citep{fireact,agenttuning,agentflan,toolbench,toolace,toolacer,agentrefine}. Action-level robustness is added by function masking (Hammer~\citep{hammer}) or critical-step selection (ATLaS~\citep{atlas}, SWE-Lego~\citep{swelego}). Across these methods the training signal is positive-only at the rendered action level: the student is told what the expert did, but not how its own likely mistake at the same state differs from it.

\paragraph{Preference optimization, self-play, and on-policy distillation.}
DPO~\citep{dpo} and its variants (IPO~\citep{ipo}, SimPO~\citep{simpo}, KTO~\citep{kto}, DPOP~\citep{dpop}, BDPO~\citep{bdpo}) optimize chosen--rejected pairs without a reward model. Multi-turn extensions move preference closer to the action level: DMPO~\citep{dmpo} length-normalizes DPO over trajectories, DiaTool-DPO~\citep{diatooldpo} adapts it to tool-augmented dialogue, Step-DPO~\citep{stepdpo} places preference at human/AI-labeled reasoning steps, ETO~\citep{eto} contrasts full success/failure trajectories, and HPL~\citep{hpl} reports failure modes at both step and trajectory level. Self-play (SPIN~\citep{spin}, T-SPIN~\citep{tspin}), AI feedback (Self-Rewarding LM~\citep{selfreward}, RLAIF~\citep{rlaif}), and on-policy distillation~\citep{opd} construct negatives from the model's own generations. Agentic-DPO is the minimal action-level variant: chosen and rejected share the same expert state, the rejected sample is one step from the current student (no per-step labels and no full-trajectory rollout), and no teacher is queried during training.

\paragraph{Reinforcement learning for agents.}
RLHF-style online RL~\citep{instructgpt,ppo,webrl,digirl,webagentr1,ragen,agentgymrl} trains agents through environment rollouts on benchmarks such as WebShop~\citep{webshop}, ALFWorld~\citep{alfworld}, and SWE-bench~\citep{swebench}, requiring reward design and rollout workers; PivotRL~\citep{pivotrl} reduces this cost via informative pivot states. Agentic-DPO is offline and only samples one-step student actions at expert states, recovering part of the self-mistake signal of RL without environment rollouts, reward models, or full-trajectory exploration.

%% ============================================================
\section{Method}
\label{sec:method}
%% ============================================================

Agentic-DPO trains an agent from expert trajectories. 
Algorithm~\ref{alg:agentic_dpo} summarizes the full procedure. 
We next describe the problem setting, the preference objective of Agentic-DPO, and the stabilization design including SFT anchor and PPA.

\begin{algorithm}[t]
\caption{Agentic-DPO}
\label{alg:agentic_dpo}
\begin{algorithmic}[1]
\Require Expert trajectories $\mathcal{D}_{\mathrm{exp}}$, schema views $\Phi$, base policy $\pi_{\theta_0}$, rounds $R$, steps-per-round $T$, batch size $B$, candidates $K$, coefficients $\beta,\alpha,\lambda$
\State SFT-warm-up $\pi_\theta$ on $\mathcal{D}_{\mathrm{exp}}$ with online PPA; freeze $\pi_{\mathrm{ref}}\leftarrow\pi_\theta$
\For{$r=1,\ldots,R$} \Comment{Refresh round}
    \State Snapshot $\theta_{\mathrm{student}}\leftarrow\theta$; sample slice $\mathcal{S}_r\!\subset\!\mathcal{D}_{\mathrm{exp}}$ of size $T\!\cdot\!B$ \Comment{just enough for this round}
    \State $\mathcal{D}_{\mathrm{pref}}^{(r)}\!\leftarrow\!\emptyset$
    \For{each $(s_t,u_t^+)\in\mathcal{S}_r$}
        \State Sample $\phi\!\sim\!\mathrm{Unif}(\Phi)$, render $s_t^\phi$, draw $K$ candidates from $\pi_{\theta_{\mathrm{student}}}(\cdot\mid s_t^\phi)$
        \State drop expert-equivalent; pick the highest-log-prob $u_t^-$ and append $(s_t,u_t^+,u_t^-)$ to $\mathcal{D}_{\mathrm{pref}}^{(r)}$
    \EndFor
    \For{$T$ gradient steps} \Comment{Model update with PPA}
        \State Minibatch from $\mathcal{D}_{\mathrm{pref}}^{(r)}$; for each triple sample $\psi\!\sim\!\mathrm{Unif}(\Phi)$ and render $(s_t^\psi,a_t^{+,\psi},a_t^{-,\psi})$
        \State Update $\theta$ with $\mathcal{L}_{\mathrm{train}}=\mathcal{L}_{\mathrm{ADPO}}+\lambda\mathcal{L}_{\mathrm{SFT}}$
    \EndFor
\EndFor
\State \Return $\pi_\theta$
\end{algorithmic}
\end{algorithm}

\subsection{Problem Setting}
\label{sec:problem}

The input to Agentic-DPO is an expert trajectory dataset
\[
\mathcal{D}_{\mathrm{exp}}
=
\{\tau_i\}_{i=1}^{N},
\qquad
\tau_i=\{(s_{i,t},u_{i,t}^+)\}_{t=1}^{T_i}.
\]
Here \(s_{i,t}\) is the agent state at step \(t\), including the system instruction, conversation history, available tools or actions, previous observations, and any task-specific context. 
The latent expert action \(u_{i,t}^+\) is the decision made by the expert policy at that state. 
We assume access only to these expert trajectories and do not execute the environment or call external model during training.

A rendered action is produced from a latent action through a schema view. 
Let \(\phi \in \Phi\) denote a schema view. 
Each view defines a state renderer $g_\phi$ and an action renderer $r_\phi$:
\[
s_t^\phi = g_\phi(s_t),
\qquad
a_t^\phi = r_\phi(u_t).
\]
For example, the same latent tool call can be rendered as a ReAct-style action or a JSON-list function call, as long as the underlying decision is unchanged. 
This distinction is important because agent actions are not plain text labels: the same latent decision may have multiple valid textual forms.

We train a student policy \(\pi_\theta(a \mid s)\), which assigns probabilities to rendered action strings. 
The SFT objective optimizes
\[
\mathcal{L}_{\mathrm{SFT}}(\theta)
=
-
\mathbb{E}_{(s_t,u_t^+),\,\phi}
\left[
\log \pi_\theta(r_\phi(u_t^+) \mid g_\phi(s_t))
\right].
\]
SFT gives the student positive examples of expert behavior, but it does not compare the expert action with actions that the current student would choose at the same state. 
Agentic-DPO adds this missing contrast by constructing preference triples $(s_t,u_t^+,u_t^-)$, where \(u_t^+\) is the expert latent action and \(u_t^-\) is a one-step negative action sampled from the current student at the same state. 
The negative is generated without taking the action in the environment, so the method stays offline.

\subsection{Agentic-DPO: State-Conditioned Action Preferences}
\label{sec:agentic_dpo}

At the beginning of each refresh round, Agentic-DPO uses the current student as the negative sampler. 
For an expert step \((s_t,u_t^+)\), we first sample a schema view \(\phi\), render the state as \(s_t^\phi\), and sample \(K=4\) one-step candidate actions from the current student: $\{\tilde a_t^{(k),\phi}\}_{k=1}^{K} \sim \pi_{\theta_{\mathrm{student}}}(\cdot \mid s_t^\phi).$
We parse each valid candidate into a latent action \(\tilde u_t^{(k)}\). 
Candidates that cannot be parsed or equivalent to the expert action are removed, where equivalence means that the latent decision is the same, such as the same tool and the same arguments after normalization. 
If no different valid candidate remains, the expert step is skipped in the current refresh round. 
This avoids turning the preference signal into a decision-irrelevant comparison, such as formatting and argument order.

Among the remaining candidates, we select the one with the highest student log probability:
\[
u_t^-
=
\arg\max_{\tilde u_t^{(k)} \neq u_t^+}
\log \pi_{\theta_{\mathrm{student}}}
\left(
r_\phi(\tilde u_t^{(k)}) \mid g_\phi(s_t)
\right).
\]
This hardest-different-action rule selects a negative that the current student is likely to make, while still requiring that the negative correspond to a different latent decision from the expert. 
The resulting triples form the preference buffer \(\mathcal{D}_{\mathrm{pref}}\).

Given a rendered preference triple \((s_t^\phi,a_t^{+,\phi},a_t^{-,\phi})\), Agentic-DPO optimizes the action-level DPO loss
\begin{equation}
\mathcal{L}_{\mathrm{ADPO}}(\theta)
=
-
\mathbb{E}_{(s_t,u_t^+,u_t^-),\,\phi}
\left[
\log \sigma
\left(
\beta_{\mathrm{eff}}(a_t^{+,\phi},a_t^{-,\phi})
\left[
\log \frac{\pi_\theta(a_t^{+,\phi} \mid s_t^\phi)}
{\pi_\theta(a_t^{-,\phi} \mid s_t^\phi)}
-
\log \frac{\pi_{\mathrm{ref}}(a_t^{+,\phi} \mid s_t^\phi)}
{\pi_{\mathrm{ref}}(a_t^{-,\phi} \mid s_t^\phi)}
\right]
\right)
\right].
\label{eq:agentic_dpo}
\end{equation}
The reference policy \(\pi_{\mathrm{ref}}\) is the frozen SFT-warmed student. 
All log probabilities are summed only over the generated action tokens, not over the state tokens. 
Because action lengths can differ across tools, GUI actions, and natural-language replies, we use a length-scaled preference coefficient
\[
\beta_{\mathrm{eff}}(a^+,a^-)
=
\frac{\beta}{\max(|a^+|,|a^-|)^\alpha},
\qquad
\alpha=0.5.
\label{eq:length_scaled_beta}
\]
Here \(|a|\) is the number of action tokens. 
The scalar \(\beta\) is fixed across runs, while the length scaling prevents long actions from dominating short actions only because they contain more tokens.

Unlike trajectory-level preference learning, both the chosen and rejected actions share the same expert state. 
The loss therefore does not ask the model to prefer one full trajectory over another. 
It asks the model to prefer the expert action over a specific wrong action that the current student is likely to take at that state.

\subsection{Stabilizing Agentic-DPO}
\label{sec:stabilization}

A direct application of Eq.~\eqref{eq:agentic_dpo} can be unstable because agent actions are represented as rendered text. 
A preference pair may differ both in the latent decision and in the schema-specific surface form. 
If this is not controlled, the DPO gradient can partly optimize for the rendering style rather than the action decision. 
Agentic-DPO uses SFT anchor and PPA.

\paragraph{SFT warm-up and anchor.}
The SFT warm-up gives the student basic format validity and expert-action coverage before any preference optimization. 
After warm-up, we freeze the reference policy \(\pi_{\mathrm{ref}}\), and then use the current student in each refresh round to generate negatives. 
During preference training, we keep an SFT anchor:
\[
\mathcal{L}_{\mathrm{train}}(\theta)
=
\mathcal{L}_{\mathrm{ADPO}}(\theta)
+
\lambda
\mathcal{L}_{\mathrm{SFT}}(\theta),
\label{eq:agentic_dpo_anchor}
\]
where \(\lambda>0\). 
The anchor keeps the expert action probability from being pushed down when the preference loss suppresses a negative action. 
It also keeps the policy close to valid action syntax while the DPO term changes the relative preference between expert and student actions.

\paragraph{Policy-Preserving Augmentation.}
PPA renders the same latent decision under multiple schema views. 
PPA is applied to both the expert action and the selected student negative, so each preference pair stays schema-consistent:
\[
(s_t^\phi,a_t^{+,\phi},a_t^{-,\phi})
=
(g_\phi(s_t),r_\phi(u_t^+),r_\phi(u_t^-)).
\]

The first PPA family is \emph{action-rendering rewrites}. 
These transformations change how actions are serialized while preserving the action choice. 
For tool-use agents, we render the same tool call in formats such as ReAct-style calls and JSON-list calls. 
We also use tool renamings, where the tool names and tool descriptions in the state are renamed consistently with the action string. 

The second PPA family is \emph{state-context rewrites}. 
These transformations change the rendered history before the current expert step while preserving the current target action. 
In recovery-style contexts, we inject an earlier erroneous action and its correction into the history, and then ask the model to choose the same current expert action. 
This creates states where the agent has already recovered from a local mistake, but the latent decision at the current step remains fixed.

During training, PPA is sampled online. 
For each preference triple in a minibatch, we sample a view \(\phi \sim \mathrm{Unif}(\Phi)\), render the state, expert action, and negative action under that same view, and compute Eq.~\eqref{eq:agentic_dpo}. 
Across updates, the same latent preference can be seen under several schemas. 
Detailed examples of PPA can be found in Appendix \ref{app:ppa_examples}.

\subsection{Why the Stabilizers Help}
\label{sec:theory}

We give two local observations that motivate the SFT anchor and PPA. Full assumptions, proofs, the hard-negative extension, and refresh-round diagnostics are given in Appendix~\ref{app:proofs}.

\paragraph{The SFT anchor is aligned with Agentic-DPO.}
Fix an expert state \(s\) with a finite latent action set \(\mathcal{A}(s)\), logits \(z_\theta(s)\), and policy \(\pi_\theta(\cdot\mid s)=\mathrm{softmax}(z_\theta(s))\). 
Let \(q_s\) be the expert action distribution, and let \(\pi_{\mathrm{ref}}\) be the frozen SFT-warmed reference policy with \(\pi_{\theta_0}=\pi_{\mathrm{ref}}\). 
For \(a^+\sim q_s\) and \(a^-\sim\pi_{\mathrm{ref}}(\cdot\mid s)\), let \(\mathcal{L}_{\mathrm{con}}\) be the per-state DPO contrastive loss and \(\mathcal{L}_{\mathrm{SFT}}\) be the per-state SFT loss.

\begin{proposition}[Local SFT-DPO alignment]
\label{thm:local_alignment}
At initialization \(\pi_{\theta_0}=\pi_{\mathrm{ref}}\), the contrastive and SFT gradients in logit space satisfy
\begin{equation}
\nabla_{z_{\theta_0}(s)}\mathcal{L}_{\mathrm{con}}(\theta_0;s)
=
\tfrac{\beta}{2}
\bigl(
\pi_{\mathrm{ref}}(\cdot\mid s)-q_s
\bigr),
\qquad
\nabla_{z_{\theta_0}(s)}\mathcal{L}_{\mathrm{SFT}}(\theta_0;s)
=
\pi_{\mathrm{ref}}(\cdot\mid s)-q_s.
\label{eq:local_alignment}
\end{equation}
Therefore, the anchored objective \(\mathcal{L}_{\mathrm{con}}+\lambda\mathcal{L}_{\mathrm{SFT}}\) has gradient
\[
\left(\lambda+\tfrac{\beta}{2}\right)
\bigl(
\pi_{\mathrm{ref}}(\cdot\mid s)-q_s
\bigr).
\]
\end{proposition}

At initialization, the SFT anchor and the Agentic-DPO contrastive term push probability mass in the same logit-space direction.
Thus the anchor does not oppose preference optimization; it reinforces the expert-directed update while preventing expert-action likelihood from being displaced. 
A sufficiently small step along the anchored gradient gives a first-order decrease of the expert KL \(V_s(z)=\mathrm{KL}(q_s\Vert\pi_z(\cdot\mid s))\), and Appendix~\ref{app:proofs} extends the statement to the hard-negative rule used in Algorithm~\ref{alg:agentic_dpo}.

\paragraph{PPA averages out schema-form gradients.}
For a policy-preserving view \(\phi\), write the stylized decomposition
\[
\log\pi_\theta(r_\phi(u)\mid g_\phi(s))
=
G_\theta(u\mid s)
+
H_\theta(\phi\mid s,u)
+
\epsilon_\theta,
\]
where \(G_\theta\) scores the latent decision, \(H_\theta\) scores schema-specific form, and \(\epsilon_\theta\) is a residual independent of \(u\) and \(\phi\).
The corresponding Agentic-DPO margin is
\[
\Delta_\theta^\phi
=
\underbrace{
G_\theta(u^+\mid s)-G_\theta(u^-\mid s)
}_{\text{latent-decision margin}}
+
\underbrace{
H_\theta(\phi\mid s,u^+)-H_\theta(\phi\mid s,u^-)
}_{\text{schema-form margin}}
+
\epsilon_\theta.
\]
With a single schema, the schema-form margin can dominate the gradient, making DPO prefer a rendering style rather than the latent expert action. 
PPA instead samples multiple views for the same latent preference pair. 
If these views are balanced so that
\[
\mathbb{E}_{\phi\sim\Phi}
\left[
\nabla_\theta
\bigl(
H_\theta(\phi\mid s,u^+)-H_\theta(\phi\mid s,u^-)
\bigr)
\right]
\approx 0,
\]
then the expected preference gradient is dominated by the shared latent-decision margin. 
This is why PPA uses both action-rendering rewrites and state-context rewrites: the surface form and local history vary across views, while the expert-over-negative latent decision stays fixed.
Together, these two observations predict that removing either the SFT anchor or PPA should produce visible regressions in seed stability and the canonical--perturbation gap, which we test in Section~\ref{sec:ablation}.

%% ============================================================
\section{Experiments}
\label{sec:experiments}
%% ============================================================

\subsection{Experimental Setup}
\label{sec:setup}

We evaluate \textbf{Agentic-DPO} on three agent benchmarks: StableToolBench~\citep{stabletoolbench} for tool use, $\tau$-bench retail~\citep{taubench} for long-horizon user-interleaved tool interaction, and Mind2Web~\citep{mind2web} for web-GUI grounding.
We report canonical and perturbed accuracy on StableToolBench (perturbation set in Appendix~\ref{app:stb_perturbed}), task success on the original released $\tau$-bench retail tasks, and average step success across the cross-task, cross-website, and cross-domain held-out splits of Mind2Web. All results are averaged over three random seeds.

Backbones are Qwen3.5-2B/4B/9B~\citep{qwen3}, plus Gemma3-4B~\citep{gemma3} in ablations. We compare against SFT, PPA+SFT, DFT~\citep{dft} (a reward-rectified SFT variant), ETO~\citep{eto} (trajectory-level DPO), and GRPO~\citep{grpo} (online RL with grouped rollouts) on a shared backbone, expert traces, and evaluation protocol per benchmark; GRPO is run only on StableToolBench and $\tau$-bench retail (Mind2Web's static traces lack the step-level reward GRPO needs). Agentic-DPO is trained on 4 A6000 GPUs with hyperparameters $\lambda=0.5$, $\beta=0.008$, $K=4$ student samples, $R=5$ negative-refresh rounds, and effective batch size $16$ across all benchmarks and backbones; the remaining loss-coefficient details and the per-round dynamics are reported in Appendix and Section~\ref{sec:refresh_dynamics}.

\subsection{Main Results Across Agent Benchmarks}
\label{sec:main_results}

Table~\ref{tab:main_results} compares Agentic-DPO with baselines across the three benchmarks and three model sizes.

\begin{table}[!t]
\centering
\caption{
Performance comparison of Agentic-DPO and baselines across agent benchmarks.
The best and second-best results within each model group are highlighted in \textbf{bold} and with an \underline{underline}.
}
\label{tab:main_results}
\renewcommand{\arraystretch}{0.6}
\resizebox{\textwidth}{!}{
\footnotesize
\begin{tabular}{lccccc}
\toprule
\multirow{2}{*}{\textbf{Method}}
&
\multicolumn{2}{c}{\textbf{StableToolBench}}
&
\textbf{$\tau$-bench Retail}
&
\textbf{Mind2Web}
&
\multirow{2}{*}{\textbf{Average}}
\\
\cmidrule(lr){2-3}
&
\textbf{Canon}
&
\textbf{Pert}
&
\textbf{Success Rate}
&
\textbf{Avg. Step Success}
&
\\
\midrule

\multicolumn{6}{l}{\textbf{Qwen3.5-2B}} \\
~~Base
& 45.2$_{\pm2.1}$ & 42.1$_{\pm2.3}$ & 12.5$_{\pm1.5}$ & 25.6$_{\pm1.8}$ & 31.4 \\
~~SFT
& 57.1$_{\pm3.2}$ & 52.3$_{\pm3.0}$ & 18.0$_{\pm2.0}$ & 35.4$_{\pm2.2}$ & 40.7 \\
~~PPA+SFT
& 83.4$_{\pm2.5}$ & 81.1$_{\pm2.8}$ & 22.1$_{\pm3.1}$ & 42.1$_{\pm2.5}$ & 57.2 \\
~~DFT
& 85.2$_{\pm6.5}$ & 79.6$_{\pm6.6}$ & 25.4$_{\pm2.8}$ & \underline{49.2}$_{\pm2.7}$ & 59.9 \\
~~ETO
& \underline{85.3}$_{\pm2.2}$ & \underline{81.3}$_{\pm2.7}$ & 28.5$_{\pm2.2}$ & 48.5$_{\pm2.4}$ & 60.9 \\
~~GRPO
& 85.0$_{\pm2.9}$ & 79.8$_{\pm4.2}$ & \underline{32.1}$_{\pm1.9}$ & -- & -- \\
\rowcolor{blue!8}
~~Agentic-DPO
& \textbf{90.9}$_{\pm1.3}$ & \textbf{85.5}$_{\pm1.5}$ & \textbf{35.2}$_{\pm1.6}$ & \textbf{52.3}$_{\pm1.8}$ & \textbf{66.0} \\

\midrule
\multicolumn{6}{l}{\textbf{Qwen3.5-4B}} \\
~~Base
& 50.5$_{\pm2.2}$ & 48.2$_{\pm2.4}$ & 14.5$_{\pm1.6}$ & 28.4$_{\pm1.9}$ & 35.4 \\
~~SFT
& 65.2$_{\pm2.8}$ & 60.5$_{\pm2.9}$ & 20.4$_{\pm2.1}$ & 38.6$_{\pm2.3}$ & 46.2 \\
~~PPA+SFT
& 82.0$_{\pm2.1}$ & 77.0$_{\pm8.5}$ & 26.5$_{\pm3.2}$ & 46.1$_{\pm2.8}$ & 57.9 \\
~~DFT
& 85.4$_{\pm1.8}$ & 81.2$_{\pm2.5}$ & 29.8$_{\pm2.1}$ & 49.0$_{\pm2.4}$ & 61.4 \\
~~ETO
& \underline{88.5}$_{\pm1.5}$ & \underline{85.1}$_{\pm2.0}$ & 32.6$_{\pm1.9}$ & \underline{50.5}$_{\pm2.2}$ & 64.2 \\
~~GRPO
& 86.4$_{\pm2.3}$ & 82.1$_{\pm5.3}$ & \underline{35.4}$_{\pm1.8}$ & -- & -- \\
\rowcolor{blue!8}
~~Agentic-DPO
& \textbf{91.3}$_{\pm1.0}$ & \textbf{89.4}$_{\pm1.2}$ & \textbf{38.5}$_{\pm1.4}$ & \textbf{54.5}$_{\pm9.0}$ & \textbf{68.4} \\

\midrule
\multicolumn{6}{l}{\textbf{Qwen3.5-9B}} \\
~~Base
& 65.2$_{\pm2.3}$ & 62.1$_{\pm2.5}$ & 16.5$_{\pm1.8}$ & 35.2$_{\pm2.1}$ & 44.8 \\
~~SFT
& 78.5$_{\pm2.5}$ & 74.2$_{\pm2.6}$ & 21.7$_{\pm2.2}$ & 45.6$_{\pm2.4}$ & 55.0 \\
~~PPA+SFT
& 92.2$_{\pm1.6}$ & 90.1$_{\pm0.7}$ & 23.8$_{\pm4.2}$ & 52.4$_{\pm2.1}$ & 64.6 \\
~~DFT
& \underline{93.0}$_{\pm1.8}$ & 90.7$_{\pm2.3}$ & 33.0$_{\pm3.0}$ & 55.6$_{\pm1.8}$ & 67.9 \\
~~ETO
& 92.5$_{\pm2.1}$ & \underline{91.1}$_{\pm1.9}$ & 20.9$_{\pm11.8}$ & \underline{56.8}$_{\pm2.5}$ & 65.5 \\
~~GRPO
& 88.6$_{\pm1.1}$ & 84.9$_{\pm3.0}$ & \underline{40.0}$_{\pm2.5}$ & -- & -- \\
\rowcolor{blue!8}
~~Agentic-DPO
& \textbf{94.1}$_{\pm0.1}$ & \textbf{92.0}$_{\pm0.1}$ & \textbf{41.4}$_{\pm3.5}$ & \textbf{64.4}$_{\pm1.2}$ & \textbf{73.0} \\

\bottomrule
\end{tabular}
}
\end{table}

Agentic-DPO is the best method on every metric and scale, with the largest gains over plain SFT (avg.\ $+25.3$/$+22.2$/$+18.0$ on 2B/4B/9B). It also consistently improves over PPA+SFT, isolating the contribution of expert-vs-student-mistake contrast beyond multi-view positives.
On StableToolBench, GRPO improves over SFT but trails ETO and Agentic-DPO at every scale, since on many states the group of $G=4$ rollouts shares the same correctness label and yields zero advantage signal.
On $\tau$-bench retail, where the agent interacts with a user simulator over long horizons, Agentic-DPO lifts Qwen3.5-9B from $21.7\%$ to $41.4\%$ success and matches GRPO ($40.0\%$) with overlapping confidence intervals at a fraction of its rollout cost (Section~\ref{sec:discussion}); ETO's $\pm11.8$ standard deviation at this scale reflects the high variance of full-trajectory preference learning, which Agentic-DPO sidesteps by anchoring each pair to a single expert state.
On Mind2Web, Agentic-DPO gives the largest relative gain at 9B, raising average step success from $45.6\%$ (SFT) and $56.8\%$ (ETO) to $64.4\%$.

\subsection{Ablation Study}
\label{sec:ablation}

Table~\ref{tab:ablation} ablates the four stabilizing components on StableToolBench for Qwen3.5-2B and Gemma3-4B. The \textbf{SFT warm-up} is load-bearing: removing it drops accuracy to $0.0\%$ on both backbones, reflecting a $100\%$ parser-rejection floor where student samples drift away from the action schema and produce ungradeable strings before any preference signal can act. With the warm-up in place, removing the \textbf{SFT anchor} during preference training partially recovers ($72.5$/$62.5$ canonical on 2B/Gemma) but trails the full recipe by $14$--$18$ points and shows higher seed variance, indicating the contrastive term alone is not enough to stabilize action-level DPO. Removing either \textbf{PPA family} keeps canonical accuracy high but widens the canonical--perturbation gap from $\sim$4 points (full recipe) to $11$--$15$ points; action-rendering PPA contributes slightly more to perturbation robustness than recovery-context PPA, and removing both collapses the gap to the same level. Only the full recipe combines high canonical accuracy with a small perturbation gap on both backbones.

\begin{table}[t]
\centering
\caption{
Ablation study on StableToolBench. 
The full Agentic-DPO recipe uses SFT warm-up, an SFT anchor during DPO training, action-rendering PPA, and recovery-context PPA.
}
\label{tab:ablation}
\renewcommand{\arraystretch}{0.9} 
\resizebox{\textwidth}{!}{
\small
\begin{tabular}{lcccccc}
\toprule
\multirow{2}{*}{\textbf{Method}}
&
\multicolumn{3}{c}{\textbf{Qwen3.5-2B}}
&
\multicolumn{3}{c}{\textbf{Gemma3-4B}}
\\
\cmidrule(lr){2-4}
\cmidrule(lr){5-7}
&
\textbf{Canonical}
&
\textbf{Perturb.}
&
\textbf{Gap}
&
\textbf{Canonical}
&
\textbf{Perturb.}
&
\textbf{Gap}
\\
\midrule
SFT
& 57.1$_{\pm3.2}$ & 52.3$_{\pm3.0}$ & 4.8
& 65.4$_{\pm3.5}$ & 50.1$_{\pm3.8}$ & 15.3 \\
PPA+SFT
& 83.4$_{\pm2.5}$ & 81.1$_{\pm2.8}$ & \phantom{0}2.3
& 72.4$_{\pm2.8}$ & 70.1$_{\pm3.1}$ & \phantom{0}2.3 \\
Agentic-DPO w/o SFT warm-up
& \phantom{0}0.0$_{\pm0.0}$ & \phantom{0}0.0$_{\pm0.0}$ & \phantom{0}0.0
& \phantom{0}0.0$_{\pm0.0}$ & \phantom{0}0.0$_{\pm0.0}$ & \phantom{0}0.0 \\
Agentic-DPO w/o SFT anchor
& 72.5$_{\pm4.2}$ & 68.2$_{\pm4.5}$ & \phantom{0}4.3
& 62.5$_{\pm4.8}$ & 58.1$_{\pm5.2}$ & \phantom{0}4.4 \\
Agentic-DPO w/o action-rendering PPA
& 88.5$_{\pm2.1}$ & 75.3$_{\pm2.9}$ & 13.2
& 75.8$_{\pm2.4}$ & 61.2$_{\pm3.3}$ & 14.6 \\
Agentic-DPO w/o recovery-context PPA
& 89.2$_{\pm1.9}$ & 78.5$_{\pm2.6}$ & 10.7
& 76.1$_{\pm2.2}$ & 65.3$_{\pm2.9}$ & 10.8 \\
Agentic-DPO w/o PPA
& 88.0$_{\pm1.4}$ & 76.6$_{\pm1.9}$ & 11.4
& 74.4$_{\pm3.1}$ & 64.0$_{\pm3.9}$ & 10.4 \\
\rowcolor{blue!8}
Full Agentic-DPO
& 90.9$_{\pm1.3}$ & 85.5$_{\pm1.5}$ & \phantom{0}5.4
& 76.5$_{\pm1.8}$ & 72.8$_{\pm2.1}$ & \phantom{0}3.7 \\
\bottomrule
\end{tabular}
}
\end{table}

\subsection{Robustness and Out-of-Distribution Generalization}
\label{sec:robustness_holdout}

We test whether the gains come from learning the training-time PPA views or from improving the latent action policy. StableToolBench-Perturbed has nine evaluation operators; two (\texttt{rename\_tool}, \texttt{format\_switch}) overlap with training-time PPA and seven are held out (full list in Appendix~\ref{app:stb_perturbed}).
As Table~\ref{tab:robustness_holdout} shows, Agentic-DPO improves both in-class and held-out accuracy with only a $2.7$-point gap between them, so the robustness gain does not collapse on perturbations absent from training-time PPA; the much smaller seed standard deviations also indicate that the anchored objective stabilizes training across seeds.
Table~\ref{tab:bfcl_ood} tests zero-shot OOD generalization: a Qwen3.5-9B agent trained on the publicly released ToolACE~\citep{toolace} corpus and evaluated on BFCL-v3~\citep{bfcl}. SFT slightly hurts the frozen base model on the Live split, while Agentic-DPO improves both Non-Live and Live scores, matching the held-out finding that PPA-stabilized preference optimization preserves OOD tool-use better than positive-only SFT.

\begin{table}[t]
\centering
\begin{minipage}[t]{0.52\linewidth}
\centering
\captionof{table}{
Held-out robustness on StableToolBench-Perturbed with Qwen3.5-2B. 
In-class uses the 2 perturbation operators related to PPA, while held-out uses the 7 operators not used by PPA.
}
\label{tab:robustness_holdout}
\renewcommand{\arraystretch}{0.95}
\small
\begin{tabular}{lcc}
\toprule
\textbf{Method} & \textbf{In-class} & \textbf{Held-out} \\
\midrule
SFT          & 53.2$_{\pm 11.5}$ & 52.3$_{\pm 15.2}$ \\
PPA+SFT      & 61.7$_{\pm 25.3}$ & 61.2$_{\pm 21.9}$ \\
\rowcolor{blue!8}
Agentic-DPO  & \textbf{89.6}$_{\pm 3.6}$ & \textbf{86.9}$_{\pm 3.0}$ \\
\bottomrule
\end{tabular}
\end{minipage}
\hfill
\begin{minipage}[t]{0.42\linewidth}
\centering
\captionof{table}{
Zero-shot BFCL-v3 evaluation for Qwen3.5-9B trained on ToolACE.
}
\label{tab:bfcl_ood}
\renewcommand{\arraystretch}{0.95}
\small
\begin{tabular}{lcc}
\toprule
\textbf{Method} & \textbf{Non-Live} & \textbf{Live} \\
\midrule
Base          & 82.66 & 56.11 \\
SFT           & 81.87 & 51.54 \\
\rowcolor{blue!8}
Agentic-DPO   & \textbf{84.17} & \textbf{57.18} \\
\bottomrule
\end{tabular}
\end{minipage}
\end{table}

\subsection{Scaling Up: Sampled Negatives and Expert Data}
\label{sec:scaling_up}

\begin{figure}[t]
\centering
\begin{minipage}[c]{0.49\linewidth}
\centering
\includegraphics[width=\linewidth]{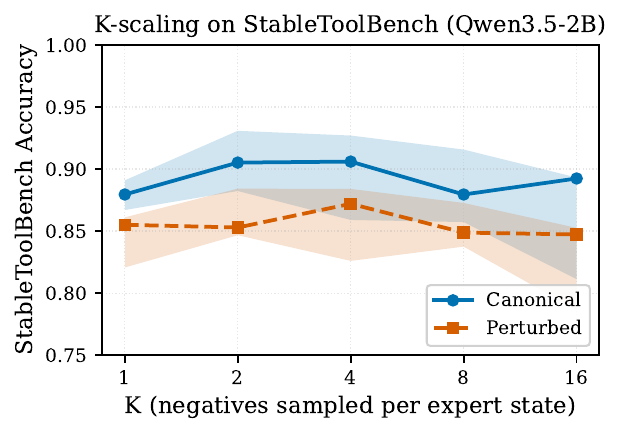}
\end{minipage}\hfill
\begin{minipage}[c]{0.49\linewidth}
\centering
\includegraphics[width=\linewidth]{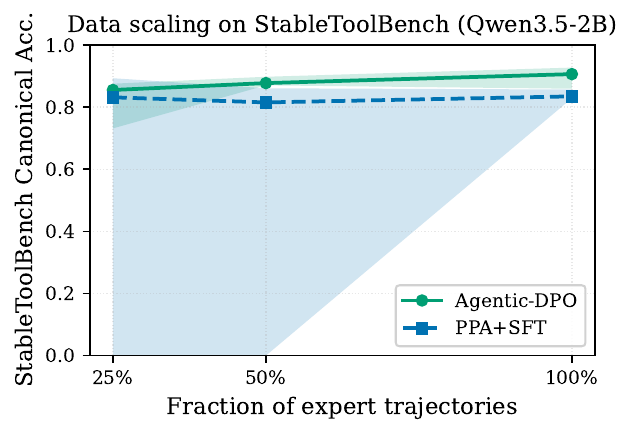}
\end{minipage}
\caption{
Scaling behavior of Agentic-DPO on StableToolBench with Qwen3.5-2B. Shaded bands show min--max. 
\textbf{Left:} increasing the number of sampled negatives \(K\) improves performance up to \(K=2\)--\(4\), after which the curve saturates. 
\textbf{Right:} Agentic-DPO is more data efficient than PPA+SFT, it has better performance and stability with low data.
}
\label{fig:scaling_up}
\end{figure}

We study two scaling axes on StableToolBench with Qwen3.5-2B: the number of one-step sampled negatives per expert state, and the fraction of expert trajectories used for training. 
The \(K\)-scaling curve in Figure~\ref{fig:scaling_up} (left) saturates quickly. 
Moving from \(K=1\) to \(K=2\)--\(4\) improves performance, but larger \(K\) gives little extra gain. This suggests that the main benefit comes from finding one plausible local mistake at each expert state, rather than from approximating a large rollout distribution. 
The data-scaling curve in Figure~\ref{fig:scaling_up} (right) shows a complementary effect: Agentic-DPO trained on only \(25\%\) of the expert trajectories already matches PPA+SFT trained on \(100\%\) of the data. And Agentic-DPO has stable performance across all data sizes, while PPA+SFT is unstable with lower than 50\% data.

\subsection{Training Dynamics: Negative-Refresh Rounds}
\label{sec:refresh_dynamics}

A refresh round resamples one-step negatives at expert states using the current student and then continues training with the same anchored objective. 
We sweep \(R\in\{0,1,\ldots,5\}\), where \(R=0\) is PPA+SFT without Agentic-DPO.
Figure~\ref{fig:refresh_dynamics} shows that the useful number of refresh rounds depends on model scale. 
The 2B model benefits from several refreshes but begins to overfit after its peak. 
The 4B model does not show the same clear overfitting pattern, but it converges more slowly. 
The 9B model converges fastest, one refresh round is enough to reach its best performance, and additional rounds mostly plateau. 
This pattern suggests that larger students need fewer negative-refresh rounds because their first set of one-step negatives is already close to the useful decision boundary.

\begin{figure}[t]
\centering
\includegraphics[width=\linewidth]{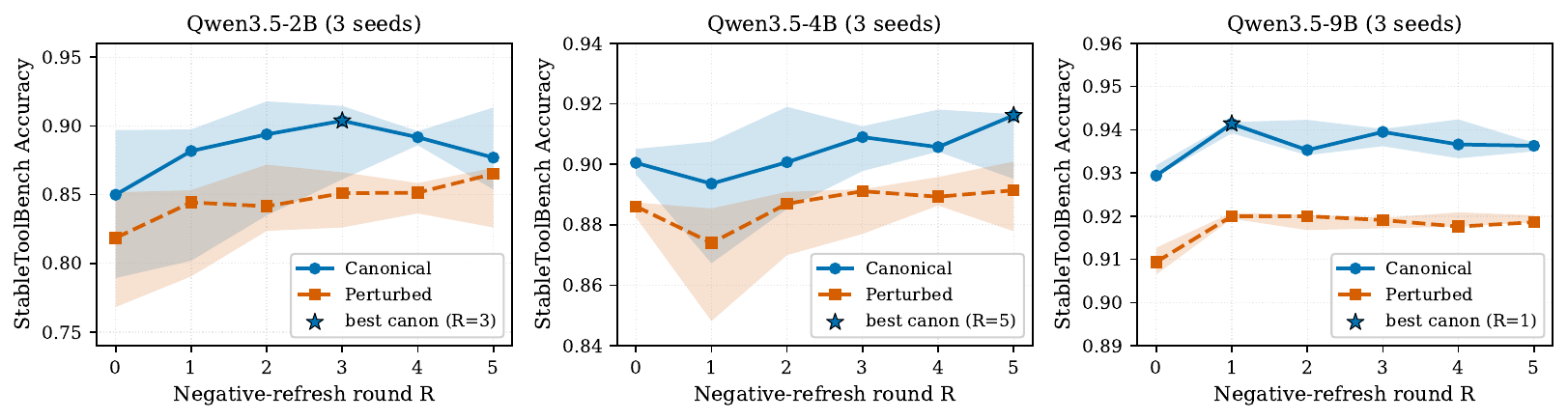}
\caption{
Effect of negative-refresh rounds \(R\) on StableToolBench for Qwen3.5-2B/4B/9B. 
Lines show per-round medians over seeds, and shaded bands show min--max. 
\(R=0\) is PPA+SFT only. 
Smaller models benefit from more refreshes, while the 9B model reaches its best performance after one refresh round.
}
\label{fig:refresh_dynamics}
\end{figure}

%% ============================================================
\section{Discussion}
\label{sec:discussion}
%% ============================================================

\textbf{Training cost.}
Agentic-DPO is more expensive than SFT, but much cheaper than online RL.
On $\tau$-bench with Qwen3.5-9B, taking SFT's per-step wall time as $1\times$, Agentic-DPO costs $1.6\times$ per step (two forwards on expert and negative per step), while GRPO costs $13.6\times$ per step because every gradient step requires fresh multi-turn rollouts.
Thus Agentic-DPO recovers a student-aware training signal while avoiding the dominant cost of RL: repeated full-trajectory environment rollouts.

\textbf{Limitations.}
Agentic-DPO is still an offline method. 
It constructs preferences only at states that already appear in the expert trajectories, so it cannot discover new states induced by the student's own long-horizon behavior. 
This is the main difference from online RL.
Agentic-DPO instead trades off this feedback for lower cost and simpler training. 
As a result, it is best suited when expert trajectories cover the important decision states, while tasks that require exploration far beyond the expert data may still benefit from online environment interaction.

%% ============================================================
\section{Conclusion}
\label{sec:conclusion}
%% ============================================================

We presented Agentic-DPO, an offline policy optimization method that converts each expert action step into a state-conditioned preference by contrasting the expert action with a one-step negative sampled from the current student, stabilized by an SFT anchor and Policy-Preserving Augmentation. Across StableToolBench, $\tau$-bench retail, and Mind2Web, Agentic-DPO improves small agents beyond SFT and strong RL/preference baselines without environment rollouts, reward-model training, or full-trajectory exploration, suggesting expert trajectories can support low-cost agentic policy optimization when their state-action structure is used directly.

%% ============================================================
%% References
%% ============================================================

\bibliographystyle{plainnat}
\bibliography{AgenticDPO}

%% ============================================================
%% Appendix
%% ============================================================
\appendix

\section{Expert Trajectory Sources}
\label{app:expert_trajectories}

Each benchmark's expert traces and the resulting Agentic-DPO training pool are summarized in Table~\ref{tab:expert_sources}.

\begin{table}[!h]
\centering
\caption{Expert trajectory sources and the resulting Agentic-DPO training pool.}
\label{tab:expert_sources}
\renewcommand{\arraystretch}{1.0}
\small
\begin{tabular}{lll}
\toprule
\textbf{Benchmark} & \textbf{Source} & \textbf{Training pool (canonical pairs)} \\
\midrule
StableToolBench~\citep{stabletoolbench} & ToolBench traces~\citep{toolbench} & $3{,}786$ \\
$\tau$-bench retail~\citep{taubench}    & released retail trajectories       & $3{,}786$ \\
Mind2Web~\citep{mind2web}               & released static action sequences   & $7{,}362$ \\
BFCL-v3 (OOD eval)~\citep{bfcl}         & ToolACE corpus~\citep{toolace} (training only) & ToolACE pool \\
\bottomrule
\end{tabular}
\end{table}

For each source, the expert action at a step is the ground-truth action shipped with the dataset: a tool call from the released trajectory (StableToolBench, $\tau$-bench, ToolACE) or a \texttt{CLICK}/\texttt{TYPE}/\texttt{SELECT} on a target element (Mind2Web). $\tau$-bench dialog turns without a tool call are mapped to a \texttt{respond} action, matching the released $\tau$-bench evaluator. BFCL-v3 is used only at evaluation time.

\section{Proofs and Supporting Analysis}
\label{app:proofs}
\label{app:theory}

This appendix provides the full statements and proofs of the two local results in Section~\ref{sec:theory}: (i) at initialization, the step-level contrastive gradient is collinear with the SFT gradient in logit space (Proposition~\ref{thm:local_alignment}); and (ii) a single anchored Agentic-DPO step produces a first-order decrease of the expert KL (Corollary~\ref{cor:local_kl}). We then extend the statement to the hard-negative rule used by Agentic-DPO.

\subsection{Local analysis: setup and statements}
\label{app:statements:local}

This analysis uses a latent-action softmax abstraction. 
It treats each rendered action as the surface form of a latent action and writes the DPO coefficient as a fixed positive scalar $\beta$. 
In the implementation, this scalar is replaced by the length-scaled $\beta_{\mathrm{eff}}$ in Eq.~\eqref{eq:length_scaled_beta}; for a fixed action pair, this only changes the local gradient scale and not its direction.

Fix an expert decision state $s$ with finite action set $\mathcal{A}(s)$. Let $z_\theta(s)\in\mathbb{R}^{|\mathcal{A}(s)|}$ denote the pre-softmax logits and $\pi_\theta(a\mid s)=\mathrm{softmax}(z_\theta(s))_a$. We write $e_a$ for the standard basis vector of action $a$. Let $\pi_{\mathrm{ref}}$ be the frozen reference policy, and assume initialization satisfies
\[
\pi_{\theta_0}(\cdot\mid s) \;=\; \pi_{\mathrm{ref}}(\cdot\mid s).
\]
Let $q_s$ denote the expert positive-action distribution at $s$ (a Dirac $\delta_{a^*(s)}$ in the deterministic trajectory setting). For each state we sample a positive $a^+\sim q_s$ and a negative $a^-\sim\pi_{\mathrm{ref}}(\cdot\mid s)$ \emph{independently}; we assume $\pi_{\mathrm{ref}}(a\mid s)>0$ on $\mathrm{supp}(q_s)$.

Define the reference-subtracted margin
\[
\Delta_\theta(s,a^+,a^-)
\;\triangleq\;
\bigl[z_\theta(a^+\!\mid\!s)-z_\theta(a^-\!\mid\!s)\bigr]
-\bigl[z_{\mathrm{ref}}(a^+\!\mid\!s)-z_{\mathrm{ref}}(a^-\!\mid\!s)\bigr],
\]
the per-pair contrastive loss $\mathcal{L}_{\mathrm{con}}(\theta;s,a^+,a^-)=-\log\sigma(\beta\,\Delta_\theta(s,a^+,a^-))$, and the per-state expected losses
\[
\bar{\mathcal{L}}_{\mathrm{con}}(\theta;s)
\;\triangleq\;\mathbb{E}_{a^+\sim q_s,\,a^-\sim\pi_{\mathrm{ref}}}\!\bigl[\mathcal{L}_{\mathrm{con}}\bigr],
\qquad
\bar{\mathcal{L}}_{\mathrm{SFT}}(\theta;s)
\;\triangleq\;\mathbb{E}_{a^+\sim q_s}\!\bigl[-\log\pi_\theta(a^+\!\mid\!s)\bigr].
\]
The joint anchored Agentic-DPO objective is $\bar{\mathcal{L}}_{\mathrm{train}}(\theta;s)\triangleq\bar{\mathcal{L}}_{\mathrm{con}}(\theta;s)+\lambda\bar{\mathcal{L}}_{\mathrm{SFT}}(\theta;s)$ with $\lambda>0$, matching Eq.~\eqref{eq:agentic_dpo_anchor} in the main paper.

\paragraph{Restatement of Proposition~\ref{thm:local_alignment} (logit-space alignment at initialization).}
Under the setup above, at $\pi_{\theta_0}=\pi_{\mathrm{ref}}$ we have
\[
\nabla_{z_{\theta_0}(s)}\bar{\mathcal{L}}_{\mathrm{con}}(\theta_0;s)
\;=\;\tfrac{\beta}{2}\bigl(\pi_{\mathrm{ref}}(\cdot\!\mid\!s)-q_s\bigr),
\qquad
\nabla_{z_{\theta_0}(s)}\bar{\mathcal{L}}_{\mathrm{SFT}}(\theta_0;s)
\;=\;\pi_{\mathrm{ref}}(\cdot\!\mid\!s)-q_s,
\]
and therefore
\[
\nabla_{z_{\theta_0}(s)}\bar{\mathcal{L}}_{\mathrm{train}}(\theta_0;s)
\;=\;
\Bigl(\lambda+\tfrac{\beta}{2}\Bigr)\bigl(\pi_{\mathrm{ref}}(\cdot\!\mid\!s)-q_s\bigr).
\]

\begin{corollary}[First-order decrease of expert KL]
\label{cor:local_kl}
Let $V_s(z)\triangleq\mathrm{KL}(q_s\,\Vert\,\pi_z(\cdot\!\mid\!s))$ and $z' = z_{\theta_0}(s) - \eta\,\nabla_{z_{\theta_0}(s)}\bar{\mathcal{L}}_{\mathrm{train}}(\theta_0;s)$ for sufficiently small $\eta>0$. Then
\[
V_s(z') \;=\; V_s(z_{\theta_0}(s))
-\eta\bigl(\lambda+\tfrac{\beta}{2}\bigr)\bigl\|\pi_{\mathrm{ref}}(\cdot\!\mid\!s)-q_s\bigr\|_2^2
+O(\eta^2).
\]
In particular, $V_s(z')<V_s(z_{\theta_0}(s))$ whenever $\pi_{\mathrm{ref}}(\cdot\!\mid\!s)\neq q_s$.
\end{corollary}

The hard-negative variant of Corollary~\ref{cor:local_kl} (Proposition~\ref{prop:hardneg}) is stated and proved in \S\ref{app:proof:hardneg}.

\subsection{Proof of Corollary~\ref{cor:local_kl}}
\label{app:proof:kl}

\begin{proof}
$V_s(z)=\mathrm{KL}(q_s\,\|\,\pi_z)$ differs from the cross-entropy $-\sum_a q_s(a)\log\pi_z(a\mid s)$ by a $z$-independent constant, so by the softmax identity $\nabla_z V_s(z)=\pi_z-q_s$. At initialization this is $\pi_{\mathrm{ref}}-q_s$. By Proposition~\ref{thm:local_alignment}, $\nabla_{z_0}\bar{\mathcal{L}}_{\mathrm{train}}=(\lambda+\beta/2)\nabla_z V_s(z_0)$. A first-order Taylor expansion of $V_s$ along $z' = z_0 - \eta\nabla_{z_0}\bar{\mathcal{L}}_{\mathrm{train}}$ gives
\[
V_s(z') \;=\; V_s(z_0) - \eta\bigl(\lambda+\tfrac{\beta}{2}\bigr)\|\pi_{\mathrm{ref}}-q_s\|_2^2 + O(\eta^2),
\]
which is strictly below $V_s(z_0)$ for sufficiently small $\eta$ whenever $\pi_{\mathrm{ref}}\neq q_s$.
\end{proof}

\subsection{Hard-Negative Extension of Corollary~\ref{cor:local_kl}}
\label{app:proof:hardneg}

We now extend the first-order KL-decrease statement to the hard-negative rule used in Agentic-DPO. 
The main text constructs negatives by sampling $K$ candidates from the current student, removing invalid candidates and candidates equivalent to the expert action, and selecting the remaining candidate with the highest student log probability. 
For the local analysis, we study this rule at initialization and in the deterministic expert setting, which matches the expert trajectories used in our experiments.

Let \(q_s=\delta_{a^*(s)}\) be a point mass on the expert action. 
Let \(\tilde q_s^-\) denote the distribution of the selected hard negative after filtering, conditioned on the event that at least one valid non-expert candidate remains. 
By construction, \(\tilde q_s^-(a^*(s))=0\).

\begin{proposition}[Hard-negative gradient at initialization]
\label{prop:hardneg}
Assume \(q_s=\delta_{a^*(s)}\), \(\tilde q_s^-(a^*(s))=0\), and \(\pi_{\theta_0}=\pi_{\mathrm{ref}}\). 
Then
\[
\nabla_{z_{\theta_0}(s)}\bar{\mathcal{L}}_{\mathrm{con}}(\theta_0;s)
=
\tfrac{\beta}{2}
\bigl(
\tilde q_s^- - \delta_{a^*(s)}
\bigr).
\]
Moreover,
\[
\left\langle
\nabla_z V_s(z_{\theta_0}(s)),
\nabla_{z_{\theta_0}(s)}\bar{\mathcal{L}}_{\mathrm{con}}(\theta_0;s)
\right\rangle
=
\tfrac{\beta}{2}
\left(
1-\pi_{\mathrm{ref}}(a^*(s)\mid s)
+
\mathbb{E}_{a^-\sim \tilde q_s^-}
[
\pi_{\mathrm{ref}}(a^-\mid s)
]
\right)
\geq 0.
\]
Therefore, for the anchored objective
\(\bar{\mathcal{L}}_{\mathrm{train}}=\bar{\mathcal{L}}_{\mathrm{con}}+\lambda\bar{\mathcal{L}}_{\mathrm{SFT}}\) with \(\lambda>0\), a sufficiently small gradient step produces a first-order decrease of \(V_s(z)=\mathrm{KL}(\delta_{a^*(s)}\Vert\pi_z(\cdot\mid s))\) whenever \(\pi_{\mathrm{ref}}(\cdot\mid s)\neq \delta_{a^*(s)}\).
\end{proposition}

\begin{proof}
At initialization \(\pi_{\theta_0}=\pi_{\mathrm{ref}}\), the reference-subtracted margin is zero for every selected pair. 
Thus the same pointwise identity used in the proof of Proposition~\ref{thm:local_alignment} applies:
\[
\nabla_{z_{\theta_0}(s)}\mathcal{L}_{\mathrm{con}}
=
-\tfrac{\beta}{2}
(e_{a^*}-e_{a^-})
=
\tfrac{\beta}{2}
(e_{a^-}-e_{a^*}).
\]
Taking expectation over \(a^-\sim\tilde q_s^-\) gives
\[
\nabla_{z_{\theta_0}(s)}\bar{\mathcal{L}}_{\mathrm{con}}(\theta_0;s)
=
\tfrac{\beta}{2}
\bigl(
\tilde q_s^- - \delta_{a^*(s)}
\bigr).
\]

For \(V_s(z)=\mathrm{KL}(\delta_{a^*(s)}\Vert\pi_z(\cdot\mid s))\), we have
\[
\nabla_z V_s(z_{\theta_0}(s))
=
\pi_{\mathrm{ref}}(\cdot\mid s)-\delta_{a^*(s)}.
\]
Therefore
\begin{align*}
\left\langle
\pi_{\mathrm{ref}}-\delta_{a^*},
\tfrac{\beta}{2}(\tilde q_s^- - \delta_{a^*})
\right\rangle
&=
\tfrac{\beta}{2}
\left(
\langle \pi_{\mathrm{ref}},\tilde q_s^-\rangle
-
\pi_{\mathrm{ref}}(a^*\mid s)
-
\tilde q_s^-(a^*)
+
1
\right) \\
&=
\tfrac{\beta}{2}
\left(
1-\pi_{\mathrm{ref}}(a^*\mid s)
+
\mathbb{E}_{a^-\sim \tilde q_s^-}
[
\pi_{\mathrm{ref}}(a^-\mid s)
]
\right)
\geq 0,
\end{align*}
where the equality uses \(\tilde q_s^-(a^*)=0\). 
The SFT part adds
\[
\lambda
\left\|
\pi_{\mathrm{ref}}(\cdot\mid s)-\delta_{a^*(s)}
\right\|_2^2
\geq 0
\]
to the inner product with the KL gradient, and it is strictly positive whenever \(\pi_{\mathrm{ref}}(\cdot\mid s)\neq\delta_{a^*(s)}\). 
A first-order Taylor expansion of \(V_s\) along the negative anchored-gradient direction then gives the claimed decrease.
\end{proof}

This proposition is intentionally stated for deterministic expert actions. 
For stochastic expert distributions, the same conclusion requires additional conditions on how the hard negative distribution relates to the support of \(q_s\), which we do not need for the trajectory setting studied in this paper.

\subsection{On the Choice of $\beta$}
\label{app:proof:autobeta}

All main experiments use a single fixed $\beta=0.008$, combined with the length-scaled coefficient $\beta_{\mathrm{eff}}(a^+,a^-)=\beta/\max(|a^+|,|a^-|)^{\alpha}$ at $\alpha=0.5$ (Eq.~\ref{eq:length_scaled_beta}). The same value is used across Qwen3.5-2B/4B/9B, Gemma3-4B, and the three benchmarks. From Proposition~\ref{thm:local_alignment}, the per-sample contrastive gradient at initialization has norm $\|\nabla_z\mathcal{L}_{\mathrm{con}}\|=\beta/\sqrt{2}$ (when $a^+\!\neq\!a^-$); with $\lambda=0.5$, this contrastive update is an order of magnitude below the SFT-anchor gradient, consistent with the SFT anchor playing a stabilizing rather than competing role.

\section{Construction of the StableToolBench-Perturbed Evaluation Set}
\label{app:stb_perturbed}

The ``Pert.'' columns in Table~\ref{tab:main_results} and Table~\ref{tab:ablation} are evaluated on a perturbed variant of StableToolBench. This appendix documents its construction and reports per-operator accuracy.

\paragraph{Perturbation classes.}
We define nine perturbation operators that act on a StableToolBench instance (system prompt + tool list + user query) without changing the latent decision needed to answer the query. Each operator is parameterized by a difficulty level \(L\!\in\!\{1,2,3\}\); a higher level applies a more aggressive transformation. Concretely, the nine operators are:
\begin{enumerate}
    \item \texttt{rename\_tool} -- rename each tool's \texttt{name} field. L1: case changes only; L2: synonym substitution; L3: aggressive rewrite.
    \item \texttt{rename\_fields} -- rename argument fields. L1: case changes; L2: synonyms; L3: synonyms plus decoy fields the model must ignore.
    \item \texttt{reorder\_args} -- shuffle argument order in the tool spec. L1: shuffle optional fields; L2: shuffle all fields and reverse tool list order; L3: merge required/optional and shuffle.
    \item \texttt{format\_switch} -- reword the system prompt and tool spec. L1: minor wording; L2: style shifts and restructure; L3: convert tool spec to OpenAI function-calling JSON, YAML, or prose.
    \item \texttt{verbosity} -- add filler context. L1: one extra sentence; L2: roughly double-length redundant context; L3: roughly five-times longer with misleading capability claims.
    \item \texttt{info\_position} -- move the relevant information within the prompt. L1: middle of the prompt; L2: end of the prompt with a long preamble; L3: bury the relevant information among unrelated content.
    \item \texttt{distractor} -- add unrelated tools. L1: 3 distractors; L2: 5 distractors from same/other categories; L3: 8 distractors with name similarity to the correct tool.
    \item \texttt{desc\_degrade} -- shorten or remove descriptions. L1: drop optional-argument descriptions; L2: drop all argument descriptions; L3: drop tool and argument descriptions.
    \item \texttt{alias\_rename} -- rename to opaque aliases. L1: light synonyms; L2: domain aliases; L3: opaque IDs (e.g., \texttt{tool\_17}, forcing pure description-based selection).
\end{enumerate}
The set is generated by the script in our code release with a single fixed seed \(\texttt{seed}\!=\!42\), and stored as one JSON file per (operator, level, original subset) combination. The full Cartesian product yields \(9\!\times\!3\!\times\!6\!=\!162\) files (the six original StableToolBench evaluation subsets G1\_category / G1\_instruction / G1\_tool / G2\_category / G2\_instruction / G3\_instruction). The aggregate ``Pert.'' number reported in the main text is the unweighted mean over the level-2 subset of these files, which we found gives a representative middle-of-distribution perturbation difficulty.

\paragraph{Relation to training-time PPA.}
StableToolBench Agentic-DPO is trained with the four-variant PPA set described in Appendix~\ref{app:ppa_examples} (\texttt{base}, \texttt{json}, \texttt{rename}, \texttt{combined}). Two evaluation operators overlap in spirit with these training views (\texttt{rename\_tool} matches \texttt{rename}; \texttt{format\_switch} matches \texttt{json}/\texttt{combined}); the other seven (\texttt{rename\_fields}, \texttt{reorder\_args}, \texttt{verbosity}, \texttt{info\_position}, \texttt{distractor}, \texttt{desc\_degrade}, \texttt{alias\_rename}) target classes outside training-time PPA, and the held-out comparison in Section~\ref{sec:robustness_holdout} aggregates exactly these seven.

\paragraph{Per-operator accuracy.}
Table~\ref{tab:perop_accuracy} reports per-operator accuracy at level $L=2$ for Qwen3.5-2B (mean over three seeds), separated into the two PPA-overlapping operators (\textit{In-class}) and the seven held-out operators.

\begin{table}[!h]
\centering
\caption{Per-operator perturbation accuracy at $L=2$ on StableToolBench-Perturbed for Qwen3.5-2B (mean over three seeds).}
\label{tab:perop_accuracy}
\renewcommand{\arraystretch}{0.9}
\footnotesize
\resizebox{\textwidth}{!}{
\begin{tabular}{lcc|ccccccc}
\toprule
& \multicolumn{2}{c|}{\textbf{In-class}} & \multicolumn{7}{c}{\textbf{Held-out}} \\
\cmidrule(lr){2-3}\cmidrule(lr){4-10}
Method & \texttt{rename\_tool} & \texttt{format\_switch} & \texttt{rename\_fields} & \texttt{reorder\_args} & \texttt{verbosity} & \texttt{info\_position} & \texttt{distractor} & \texttt{desc\_degrade} & \texttt{alias\_rename} \\
\midrule
SFT          & 54.1 & 52.3 & 53.6 & 55.7 & 53.2 & 51.8 & 49.2 & 52.7 & 49.4 \\
PPA+SFT      & 62.5 & 60.9 & 61.7 & 62.4 & 61.5 & 60.7 & 59.0 & 62.1 & 61.0 \\
\rowcolor{blue!8}
Agentic-DPO  & \textbf{90.4} & \textbf{88.7} & \textbf{88.2} & \textbf{87.9} & \textbf{87.1} & \textbf{86.4} & \textbf{84.5} & \textbf{87.6} & \textbf{86.6} \\
\bottomrule
\end{tabular}
}
\end{table}

\section{Policy-Preserving Augmentation: Worked Examples}
\label{app:ppa_examples}

This appendix gives concrete worked examples of each Policy-Preserving Augmentation variant on $\tau$-bench retail and Mind2Web, together with the recovery-context PPA used by both.

\subsection{Action-Rendering Variants}
\label{app:ppa_examples:action}

For each canonical pair $(s_t, a_t^+)$, we draw a view $\phi\!\sim\!\mathrm{Unif}(\Phi)$ from a four-element set $\Phi=\{\texttt{base},\texttt{json},\texttt{rename},\texttt{combined}\}$. The rendered state $s_t^\phi$ uses the corresponding tool spec and chat format; the rendered expert action $a_t^{+,\phi}$ is the same latent decision under the chosen schema. We illustrate by rendering one canonical $\tau$-bench retail pair and one canonical Mind2Web pair under all four views.

\paragraph{$\tau$-bench retail example.}
Canonical pair: expert tool \texttt{find\_user\_id\_by\_name\_zip} with arguments \texttt{\{"first\_name":"Yusuf","last\_name":"Rossi","zip":"19122"\}}.

\begin{itemize}
    \item \textbf{base} (ReAct, canonical tool names):
    \begin{verbatim}
Action:
{"name": "find_user_id_by_name_zip",
 "arguments": {"first_name": "Yusuf", "last_name": "Rossi", "zip": "19122"}}
    \end{verbatim}
    \item \textbf{json} (Hermes-style \verb|<tool_call>| tags, canonical tool names):
    \begin{verbatim}
<tool_call>
{"name": "find_user_id_by_name_zip",
 "arguments": {"first_name": "Yusuf", "last_name": "Rossi", "zip": "19122"}}
</tool_call>
    \end{verbatim}
    \item \textbf{rename} (ReAct, renamed tool from \texttt{TOOL\_RENAME\_MAP}):
    \begin{verbatim}
Action:
{"name": "lookup_customer_by_name_zip",
 "arguments": {"first_name": "Yusuf", "last_name": "Rossi", "zip": "19122"}}
    \end{verbatim}
    \item \textbf{combined} (Hermes \verb|<tool_call>| tags + renamed tool):
    \begin{verbatim}
<tool_call>
{"name": "lookup_customer_by_name_zip",
 "arguments": {"first_name": "Yusuf", "last_name": "Rossi", "zip": "19122"}}
</tool_call>
    \end{verbatim}
\end{itemize}
The state messages and tool spec headers used by the system prompt are rewritten consistently with each view, so the \texttt{rename} and \texttt{combined} variants present a tool spec with the renamed tool names; the model never sees the canonical name during a \texttt{rename}-rendered step.

\paragraph{Mind2Web example.}
Canonical pair: expert operation \texttt{SELECT} with arguments \texttt{\{"choice":"C","value":"Pickup"\}}.

\begin{itemize}
    \item \textbf{base} (ReAct, canonical operation names):
    \begin{verbatim}
Action:
{"name": "SELECT", "arguments": {"choice": "C", "value": "Pickup"}}
    \end{verbatim}
    \item \textbf{json} (Hermes-style \verb|<tool_call>| tags, canonical operation names):
    \begin{verbatim}
<tool_call>
{"name": "SELECT", "arguments": {"choice": "C", "value": "Pickup"}}
</tool_call>
    \end{verbatim}
    \item \textbf{rename} (ReAct, renamed operation):
    \begin{verbatim}
Action:
{"name": "pick_option", "arguments": {"choice": "C", "value": "Pickup"}}
    \end{verbatim}
    \item \textbf{combined} (Hermes \verb|<tool_call>| tags + renamed operation):
    \begin{verbatim}
<tool_call>
{"name": "pick_option", "arguments": {"choice": "C", "value": "Pickup"}}
</tool_call>
    \end{verbatim}
\end{itemize}

\paragraph{Rename maps.}
The renaming used by the \texttt{rename} and \texttt{combined} variants is a single fixed dictionary per benchmark, applied consistently across the tool spec and the expert action so that the latent decision is preserved.

\begin{itemize}
    \item $\tau$-bench retail: \texttt{TOOL\_RENAME\_MAP} renames the 16 retail tools (the dialog-only \texttt{respond} action is never renamed). A 5-of-16 slice:
    \begin{verbatim}
find_user_id_by_email        -> lookup_customer_by_email
find_user_id_by_name_zip     -> lookup_customer_by_name_zip
get_user_details             -> fetch_user_profile
get_order_details            -> fetch_order_info
get_product_details          -> fetch_product_info
    \end{verbatim}
    \item Mind2Web: the operation set is small (3 entries), and the rename is total:
    \begin{verbatim}
CLICK  -> tap
TYPE   -> input_text
SELECT -> pick_option
    \end{verbatim}
\end{itemize}

\subsection{Recovery-Context Variant}
\label{app:ppa_examples:recovery}

The recovery-context variant inserts an erroneous earlier action and its correction into the trajectory's \texttt{state\_messages} \emph{before} a randomly chosen expert step. The corrected step then becomes the positive at training time, simulating the situation in which the agent had previously taken a wrong action and must recover. The injected wrong action is always a legal tool from the trajectory's own toolset (never a hallucinated name), which avoids breaking the single-turn exact-match evaluation.

\paragraph{Worked example (state messages slice).}
Suppose the canonical trajectory contains the slice:
\begin{verbatim}
[2] assistant: Action: find_user
                Action Input: {"name": "Alice"}
[3] user:      Observation: {"user_id": "123"}
[4] assistant: Action: get_order
                Action Input: {"user_id": "123"}
\end{verbatim}
After applying recovery-context PPA at index 4, the trajectory becomes:
\begin{verbatim}
[2] assistant: Action: find_user
                Action Input: {"name": "Alice"}
[3] user:      Observation: {"user_id": "123"}
[4] assistant: Thought: Let me try using get_product...               [INJECTED]
                Action: get_product
                Action Input: {"user_id": "123"}
[5] user:      Observation: {"error": "Tool returned no results"}     [INJECTED]
[6] assistant: Thought: The previous tool call failed. Let me try
                a different approach.                                 [RECOVERY]
                Action: get_order
                Action Input: {"user_id": "123"}
\end{verbatim}
Indices 4 and 5 are the injected mistake and its observation; index 6 is the original expert action with a prepended recovery thought, now serving as the expert positive at this step. The latent decision (call \texttt{get\_order} with the user id from the previous step) is preserved.

\paragraph{Why this is policy-preserving.}
The decision the model is asked to make at the rendered step (call \texttt{get\_order} on \texttt{user\_id: 123}) is identical in the canonical and recovery-context views. The difference is only in the conversation history that precedes this state. As long as the injected wrong action draws from the trajectory's legal toolset, the recovery-context view does not change the set of valid expert actions at the current state; it only changes the surface context the model conditions on.

%% ============================================================
%% NeurIPS Paper Checklist (existing)
%% ============================================================

% \input{checklist}

\end{document}